\documentclass{article}
\usepackage{ijcai26}

\usepackage{times}
\usepackage{soul}
\usepackage{url}
\usepackage[hidelinks]{hyperref}
\usepackage[utf8]{inputenc}
\usepackage[small]{caption}
\usepackage{graphicx}
\usepackage{amsmath}
\usepackage{amssymb}
\usepackage{amsthm}
\usepackage{booktabs}
\usepackage{algorithm}
\usepackage{algorithmic}
\usepackage{tikz}
\usetikzlibrary{arrows.meta,decorations.pathreplacing,positioning,calc}
\usepackage{pgfplots}
\pgfplotsset{compat=1.16}
\usepackage{xcolor}

\theoremstyle{plain}
\newtheorem{theorem}{Theorem}
\newtheorem{proposition}{Proposition}

\theoremstyle{definition}
\newtheorem{definition}{Definition}
\newtheorem{assumption}{Assumption}

\theoremstyle{remark}
\newtheorem{remark}{Remark}

\newcommand{\E}{\mathbb{E}}
\newcommand{\Wass}{W_1}
\newcommand{\Hent}{\mathcal{H}}
\newcommand{\nominal}{\bar{P}_b}
\newcommand{\amb}{\mathcal{U}}
\newcommand{\Real}{\mathbb{R}}
\newcommand{\Sbr}{s_{\mathrm{br}}}
\newcommand{\SG}{s_{G}}
\newcommand{\SF}{s_{F}}
\newcommand{\Vmm}{V_{\mathrm{maximin}}}
\newcommand{\Vbr}{V_{\mathrm{BR}}}
\newcommand{\CVaR}{\mathrm{CVaR}}
\newcommand{\EVaR}{\mathrm{EVaR}}

\title{Quantifying Risk Under Evolving Uncertainty:\\ Belief-Dependent Robustness for Safe Sequential Decision Making%
\thanks{Deep Kumar Ganguly is supported by the German Research Foundation (DFG) through Research Training Group GRK 2428 ConVeY. Jan K\v{r}et\'{\i}nsk\'y is one of the PI's of the RobustiAI project.}}

\author{
Deep Kumar Ganguly$^1$\and
Jan K\v{r}et\'{\i}nsk\'y$^{1,2}$\\
\affiliations
$^1$Technical University of Munich, Germany\\
$^2$Masaryk University, Brno, Czech Republic\\
\emails
deep.ganguly@tum.de,
jan.kretinsky@tum.de
}

\begin{document}

\maketitle

\begin{abstract}
Many agents must act while still learning what environment they are in, which raises a basic question: how cautious should they be? A fixed answer is rarely right---too much caution wastes opportunity once the environment is understood, while planning for the average case can be unsafe early on. We propose RATTL (Risk-Adversarial Total-Reward Learning), a framework that ties an agent's caution to how much it still does not know. The agent keeps a Bayesian belief about the unknown parts of its environment and hedges against a range of plausible dynamics whose size grows with the uncertainty of that belief, measured through a Wasserstein (optimal-transport) distance; as it learns, this range shrinks and behaviour shifts smoothly from worst-case caution toward ordinary reward maximization. The idea follows the Entropic Value-at-Risk, which recasts ``how cautious should I be?'' as ``how large should my set of plausible models be?''. We put the framework on a formal footing: the planning problem is well posed, and its value always sits between a fully cautious baseline and the best one could do with full knowledge---a \emph{Safety Sandwich}---with the gap closing as the environment is identified. We also begin to pin down what kind of risk this caution encodes, showing that in a canonical safety setting it matches a familiar tail-risk measure (Conditional Value-at-Risk) whose severity is set by the belief entropy. A simple worked example shows the agent holding back until it is confident, then switching to the efficient action at a clear threshold. RATTL targets runtime safety for agents---including LLM-based systems---that must act under uncertainty.
\end{abstract}

\section{The Problem: Risk That Changes as You Learn}

Consider an autonomous system that must act in real time while gradually identifying its environment: a self-driving car facing a driver of unclear intent, a medical AI choosing a treatment before a diagnostic returns, or an LLM-based agent invoking tools against an adversary of unknown sophistication. The agent holds a \emph{belief} over hidden environmental parameters and updates it through observation, yet must commit to actions \emph{now}. How much risk should it tolerate at each moment? This depends on how much uncertainty remains: early on, even a small probability of catastrophe warrants caution; once the environment is largely identified, risk-neutral optimization is appropriate. \textbf{The agent's risk attitude should be a function of its epistemic state}, decreasing in conservatism as information accumulates. Neither classical robust control (uniform worst-case reasoning) nor Bayesian RL (expected value under a diffuse belief) does this. A principled mechanism is needed that continuously adjusts the agent's position on the \emph{risk spectrum} (Figure~\ref{fig:spectrum}).

\begin{figure}[t]
\centering
\resizebox{\columnwidth}{!}{%
\begin{tikzpicture}[x=1cm,y=1cm]
  \draw[-{Latex[length=2.2mm]}, thick] (-0.1,0) -- (7.1,0);
  \foreach \x/\lbl/\sub in {%
      0/{$\mathbb{E}[X]$}/{risk-neutral},
      2.2/{$\mathrm{CVaR}_\alpha$}/{tail avg.},
      4.4/{$\mathrm{EVaR}_\alpha$}/{entropic},
      6.6/{$\mathrm{ess\,sup}(X)$}/{worst case}}{
    \fill[blue!60!black] (\x,0) circle (1.3pt);
    \node[above=1pt, font=\scriptsize] at (\x,0) {\lbl};
    \node[below=2pt, font=\scriptsize, gray] at (\x,0) {\sub};
  }
  \foreach \x in {1.1,3.3,5.5}{%
    \node[font=\tiny, above=8pt] at (\x,0) {$\leq$};
  }
  \node[right, font=\scriptsize, gray] at (7.1,0) {\emph{conservatism}};
  \draw[decorate, decoration={brace, mirror, amplitude=4pt},
        red!70!black, thick] (0,-0.7) -- (6.6,-0.7);
  \node[below, font=\scriptsize, red!70!black] at (3.3,-0.85)
    {RATTL: belief entropy selects position on this spectrum};
\end{tikzpicture}}
\caption{The coherent-risk hierarchy: EVaR is the single-parameter coherent family sweeping $\mathbb{E}$ to $\mathrm{ess\,sup}$. RATTL uses belief entropy $\Hent(b)$ to slide along it.}
\label{fig:spectrum}
\end{figure}
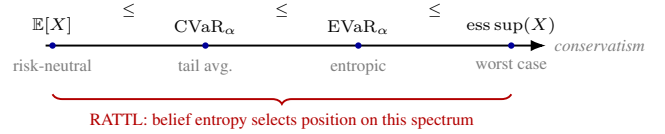

\paragraph{Contributions.} (i) We formalize RATTL: a reduction of a partially-observed turn-based stochastic game to a \emph{subjective robust MDP} whose Wasserstein ambiguity radius equals the Shannon entropy of the Bayesian belief (\S\ref{sec:rattl}). (ii) We prove three guarantees---contractivity, a Safety Sandwich, and convergence to best response---each with the precise, adversarially-checked conditions under which it holds (\S\ref{sec:theory}). (iii) We make first progress on the risk-measure identity of Wasserstein ambiguity: the inner robust value is coherent and equals a Lipschitz-regularized expectation, and on the two-point catastrophe it is \emph{exactly} a CVaR with an entropy-controlled tail level (\S\ref{sec:wevar}). (iv) We give a complete worked example with a sharp safety switch (\S\ref{sec:bridge}).

\section{The Risk Spectrum and Why EVaR}
\label{sec:evar}

For a random cost $X$, the standard coherent risk measures form a hierarchy of increasing conservatism~\cite{ahmadi2012evar}: $\mathbb{E}[X] \leq \mathrm{CVaR}_\alpha(X) \leq \mathrm{EVaR}_\alpha(X) \leq \mathrm{ess\,sup}(X)$. The Entropic Value-at-Risk is the tightest upper bound on CVaR obtainable from exponential moments:
\begin{equation}
\mathrm{EVaR}_\alpha(X) = \inf_{t > 0}\!\left\{ \tfrac{1}{t}\ln\tfrac{\mathbb{E}[e^{tX}]}{\alpha} \right\}.
\end{equation}
Two precise properties make EVaR the natural \emph{motivating} dial. First, it is a single-parameter \emph{coherent} family with \textbf{full spectrum coverage}: $\mathrm{EVaR}_1 = \mathbb{E}$ and $\mathrm{EVaR}_\alpha \to \mathrm{ess\,sup}$ as $\alpha\to 0^+$ (we use it for this coverage and its KL-DRO dual, not as the only such family). Second, it has an exact \textbf{DRO dual},
\begin{equation}
\mathrm{EVaR}_\alpha(X) = \sup_{\,D_{\mathrm{KL}}(Q\|P) \,\leq\, -\ln\alpha} \mathbb{E}_Q[X],
\label{eq:evardual}
\end{equation}
so choosing a confidence level $\alpha$ is \emph{equivalent} to choosing a KL-ball radius $-\ln\alpha$ around the nominal $P$. This is the equivalence we exploit: it converts ``how conservative should the agent be?'' into ``how large should the ambiguity set be?''---and the latter has a natural answer in the agent's epistemic uncertainty. We are careful to claim only this: EVaR supplies the conceptual bridge. RATTL's operational ambiguity sets are Wasserstein, for a safety reason developed in \S\ref{sec:wevar}, and the precise risk measure they induce is the subject of that section.

\section{RATTL: Belief Entropy as a Risk Dial}
\label{sec:rattl}

\paragraph{Setting.} A partially-observed turn-based stochastic game (PO-TBSG) has finite states $\mathcal{S}$, finite actions $\mathcal{A}$, a finite set of opponent types $\mathcal{Z}$, a terminal set $\mathcal{T}\subset\mathcal{S}$, and for each type $z$ a transition kernel $P(\cdot\mid s,a,z)\in\Delta(\mathcal{S})$. Rewards $r(s,a)$ are bounded and terminal rewards $r(s)$ are defined for $s\in\mathcal{T}$. We use the undiscounted total-reward criterion (a proper / stochastic-shortest-path MDP). The agent maintains a Bayesian belief $b\in\Delta(\mathcal{Z})$ over the type, updated by
\begin{equation}
\psi(b,s,a,s')(z) = \frac{P(s'\mid s,a,z)\,b(z)}{\sum_{z'}P(s'\mid s,a,z')\,b(z')},
\end{equation}
and the belief-averaged (nominal) kernel $\nominal(\cdot\mid s,a) = \sum_z b(z) P(\cdot\mid s,a,z)$.

\begin{definition}[Entropy-modulated ambiguity set]
For a ground metric $d$ on $\mathcal{S}$ and risk-sensitivity $\beta>0$,
\begin{equation}
\amb(b\mid s,a) = \big\{\, Q \in \Delta(\mathcal{S}) : \Wass\!\big(Q,\nominal(\cdot\mid s,a)\big) \leq \beta\,\Hent(b) \,\big\},
\end{equation}
where $\Wass$ is the $1$-Wasserstein distance and $\Hent(b)=-\sum_z b(z)\ln b(z)$ is the Shannon entropy.
\end{definition}

The radius $\varepsilon(b)=\beta\Hent(b)$ \emph{implicitly selects} the agent's conservatism: a diffuse belief gives wide ambiguity (effectively worst-case), a sharp belief gives tight ambiguity (effectively risk-neutral), with graduated conservatism in between (Figure~\ref{fig:balls}). The induced \emph{Nash-Robust Bellman operator} on bounded $V:\mathcal{S}\times\Delta(\mathcal{Z})\to\Real$ with $V(s,\cdot)=r(s)$ for $s\in\mathcal{T}$ is
\begin{equation}
\resizebox{0.92\linewidth}{!}{$\displaystyle
(\mathbb{T} V)(s,b) = \max_a\!\Big[ r(s,a) + \!\!\inf_{Q \in \amb(b\mid s,a)}\!\! \sum_{s'} Q(s')\,V\!\big(s',\psi(b,s,a,s')\big) \Big]$}.
\label{eq:bellman}
\end{equation}
The belief update $\psi$ is applied \emph{inside} the expectation, coupling the adversary's kernel choice with the agent's future information state. The operator has a game reading: the agent (Max) picks an action; a fictitious adversary (Nature/Min) picks the worst kernel within the current ambiguity budget. RATTL goes in the \emph{opposite direction} to the known RMDP\,$\to$\,stochastic-game reduction~\cite{chatterjee2024robust}: rather than expanding an RMDP into a game, we compress a partially-observed game into a subjective robust MDP whose uncertainty set is \emph{non-stationary}, evolving with belief.

\begin{figure}[t]
\centering
\begin{tikzpicture}[x=1cm,y=1cm]
  \foreach \xs/\rad/\bvec/\Hval/\col in {%
      0/1.0/{$b{=}(.5,.5)$}/{$\Hent{=}0.69$}/red!70!black,
      3.0/0.55/{$b{=}(.85,.15)$}/{$\Hent{=}0.42$}/orange!80!black,
      6.0/0.18/{$b{=}(.99,.01)$}/{$\Hent{=}0.06$}/green!50!black}{
    \draw[thick] (\xs-1.1,-1.0) -- (\xs+1.1,-1.0) -- (\xs,0.9) -- cycle;
    \node[font=\tiny,below=2pt] at (\xs-1.1,-1.0) {$\SF$};
    \node[font=\tiny,below=2pt] at (\xs+1.1,-1.0) {$\Sbr$};
    \node[font=\tiny,above=2pt] at (\xs,0.9) {$\SG$};
    \fill[\col, fill opacity=0.18] (\xs+0.15,0.0) circle (\rad);
    \draw[\col, thick] (\xs+0.15,0.0) circle (\rad);
    \fill[\col] (\xs+0.15,0.0) circle (1.2pt);
    \node[font=\scriptsize,\col,above=18pt] at (\xs,0.9) {\bvec};
    \node[font=\scriptsize,\col,above=8pt]  at (\xs,0.9) {\Hval};
  }
\end{tikzpicture}
\caption{Wasserstein ambiguity balls $\amb(b)$ on the next-state simplex for three beliefs of decreasing entropy. The ball shrinks (and its center $\nominal$ moves) as the belief sharpens: high entropy lets the adversary push mass toward catastrophic states; low entropy renders it nearly powerless.}
\label{fig:balls}
\end{figure}
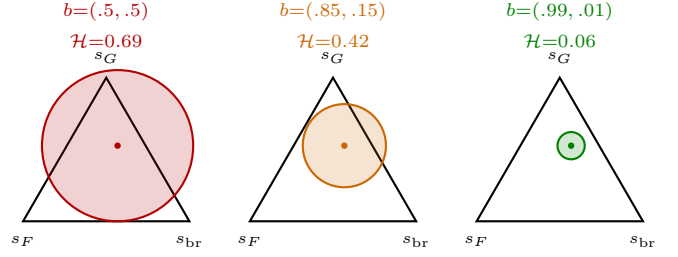

\paragraph{Assumptions.} We isolate the conditions our theorems need.
\begin{assumption}[Properness]\label{ass:proper}
Under every policy and every kernel selection in every $\amb(b)$, the terminal set $\mathcal{T}$ is reached with probability $1$.
\end{assumption}
\begin{assumption}[Uniform reachability]\label{ass:unif}
There exist $m\ge1$ and $\eta\in(0,1]$ such that from every non-terminal $(s,b)$, under every admissible control and adversary kernel, $\mathcal{T}$ is reached within $m$ steps with probability $\ge\eta$.
\end{assumption}
\begin{assumption}[Identifiability]\label{ass:ident}
For all $z\neq z'$ there is an $(s,a)$ with $P(\cdot\mid s,a,z)\neq P(\cdot\mid s,a,z')$.
\end{assumption}
\noindent Assumption~\ref{ass:unif} is the quantitative strengthening that the belief \emph{continuum} forces: a.s.\ reachability alone (Ass.~\ref{ass:proper}) need not give a uniformly bounded expected hitting time over $\Delta(\mathcal{Z})$. With finite $\mathcal{S},\mathcal{A},\mathcal{Z}$ it holds whenever the one-step absorption probability is bounded below over the (compact) action$\times$adversary sets. Rewards are assumed bounded throughout.

\section{Theoretical Guarantees}
\label{sec:theory}

\begin{theorem}[Contractivity]\label{thm:contraction}
Under Assumptions~\ref{ass:proper}--\ref{ass:unif}, let $w(s)=\sup_{b}\sup_{\mathrm{adm.}}\E^{(s,b)}[\tau_{\mathcal{T}}]$ be the worst-case expected hitting time. Then $1\le w(s)\le W:=m/\eta<\infty$, and $\mathbb{T}$ is a contraction of modulus $\rho=1-1/W<1$ in the weighted sup-norm $\|V\|_w=\max_{s\notin\mathcal{T},b}|V(s,b)|/w(s)$ on the affine space $\{V:V(s,\cdot)=r(s)\ \forall s\in\mathcal{T}\}$. Hence $\mathbb{T}$ has a unique fixed point $V^*$ and $\mathbb{T}^k V\to V^*$ geometrically.
\end{theorem}
\begin{proof}[Proof sketch]
Iterating Assumption~\ref{ass:unif} over blocks of $m$ steps gives $\Pr(\tau_{\mathcal{T}}>km)\le(1-\eta)^k$, so $w(s)\le m/\eta$ uniformly in $b$ and strategy; thus $\|\cdot\|_w$ is a genuine norm on the affine space (differences vanish on $\mathcal{T}$) equivalent to $\|\cdot\|_\infty$, which is complete. The inner $\inf$ over a fixed set is non-expansive, $|\inf_Q f-\inf_Q g|\le\sup_Q|f-g|$, and since $r(s,a)$ cancels and $\max_a$ is non-expansive, a one-step drift inequality $(\mathcal{L}w)(s)\le w(s)-1$ (worst-case unit-cost SSP, \cite{bertsekas1996neuro}) yields $|(\mathbb{T}V-\mathbb{T}V')(s,b)|\le\|V-V'\|_w\,(w(s)-1)$; dividing by $w(s)$ and using $w\le W$ gives modulus $1-1/W$. Belief augmentation is harmless: $\psi$ is a deterministic function of the history, rewards depend only on $(s,a)$, and $\mathcal{T}\subset\mathcal{S}$ so reachability is independent of $b$. Full proof in Appendix~\ref{app:contraction}.
\end{proof}

The next result is the one reviewers asked us to make concrete: what does robustness-with-learning \emph{do}, behaviorally? It brackets the value between two reference policies.

\begin{theorem}[Safety Sandwich]\label{thm:sandwich}
Let $\Vmm$ be the fixed point of the operator identical to \eqref{eq:bellman} but with the radius \emph{frozen} at $\varepsilon_{\max}=\beta\ln|\mathcal{Z}|$ (same center $\nominal$), and let $\Vbr(\cdot,z)$ be the optimal value of the non-robust MDP with known type $z$. Under Assumptions~\ref{ass:proper}--\ref{ass:ident},
\begin{equation}
\Vmm(s,b)\ \le\ V^*(s,b)\ \le\ \sum_{z\in\mathcal{Z}} b(z)\,\Vbr(s,z)\quad\forall (s,b).
\end{equation}
\end{theorem}
\begin{proof}[Proof]
\emph{Lower bound.} Since $\Hent(b)\le\ln|\mathcal{Z}|$ and both balls share the center $\nominal$, $\amb(b)\subseteq\amb_{\max}(b)$; the $\inf$ over the larger set is no larger, so $\mathbb{T}_{\max}V\le\mathbb{T}V$ pointwise for every $V$. Starting value iteration from $V^*$ and using monotonicity, $\mathbb{T}_{\max}^k V^*\le\mathbb{T}V^*=V^*$ for all $k$; letting $k\to\infty$ gives $\Vmm\le V^*$.
\emph{Upper bound.} For any policy $\pi$, the nominal kernel is feasible ($\nominal\in\amb(b)$), so the robust value $V^\pi_{\mathrm{rob}}\le V^\pi_{\mathrm{nom}}$. Running $\nominal$ with $\psi$-updates is exactly the marginal-and-posterior factorization of the mixture process ``draw $z\sim b$ once, then follow $P(\cdot\mid\cdot,\cdot,z)$''; hence the trajectory laws coincide and $V^\pi_{\mathrm{nom}}(s,b)=\sum_z b(z) V^\pi_z(s)$. As $V^\pi_z(s)\le\Vbr(s,z)$ for every type, $V^\pi_{\mathrm{rob}}(s,b)\le\sum_z b(z)\Vbr(s,z)$; taking $\sup_\pi$ gives the claim. Full proof in Appendix~\ref{app:sandwich}.
\end{proof}

\begin{remark}[Why the belief-averaged ceiling]
The stronger ceiling $V^*\le\Vbr(s,z^*)$ at the realized true type $z^*$ is \emph{false} in general: $V^*$ is a deterministic function of $(s,b)$ while $z^*$ is random, and the true kernel $P(\cdot\mid s,a,z^*)$ typically lies \emph{outside} $\amb(b)$ when $b$ is not a point mass (Identifiability makes the types $\Wass$-separated). As $b\to\delta_{z^*}$ both bounds coincide: the radius vanishes and $\sum_z b(z)\Vbr(s,z)\to\Vbr(s,z^*)$. This is exactly the asymptotic role of Theorem~\ref{thm:convergence}.
\end{remark}

\begin{theorem}[Convergence to best response]\label{thm:convergence}
Assume \ref{ass:proper}, \ref{ass:ident}, and \emph{persistent identification} (PI): along the realized trajectory, for each pair of types an identifying $(s,a)$ is visited infinitely often a.s.\ (automatic under any proper, fully-exploring behavior policy). Then $b_t\to\delta_{z^*}$ a.s.~\cite{schwartz1965bayes}, so $\Hent(b_t)\to0$ and $\amb(b_t)$ collapses to $\{P(\cdot\mid\cdot,\cdot,z^*)\}$ in Hausdorff distance; consequently $V^*(s,b_t)\to\Vbr(s,z^*)$ a.s. If in addition the per-step log-likelihood separation is bounded below (persistent excitation; e.g.\ all positive transition probabilities $\ge p_{\min}>0$), then $\E[\Hent(b_t)]=O(|\mathcal{Z}|\log t/t)$ and the price of robustness obeys $|\Vbr(s,z^*)-V^*(s,b_t)|=O(\log t/t)$.
\end{theorem}
\begin{proof}[Proof sketch]
PI supplies the excitation that Identifiability alone lacks, giving Bayesian consistency (Doob/Schwartz). The value-continuity step is the delicate one: $V^*(\cdot,b)$ is \emph{not} the fixed point of any per-belief operator because $\psi$ shifts the belief; instead one shows $\mathbb{T}$ maps the class of value functions with belief-modulus $\le L$ near $\delta_{z^*}$ into itself, provided the Bayes normalizer is bounded below on the realized support (so $\psi$ is locally Lipschitz), and the unique fixed point inherits the modulus---hence continuity at $\delta_{z^*}$. The rate follows from $\E[\Hent(b_t)]=O(\log t/t)$ under persistent excitation. Full statement and proof in Appendix~\ref{app:convergence}.
\end{proof}

\paragraph{Tractability.} The inner $\inf$ is a finite linear program. By Kantorovich--Rubinstein duality~\cite{villani2009optimal},
\begin{equation}
\resizebox{0.92\linewidth}{!}{$\displaystyle
\inf_{Q\in\amb(b)}\!\sum_{s'}\!Q(s')V(s') = \sup_{\lambda\ge0}\Big\{\E_{\nominal}\!\big[\textstyle\min_{s'}(V(s')+\lambda d(s',\cdot))\big]-\lambda\varepsilon(b)\Big\}$},
\label{eq:dual}
\end{equation}
so belief-augmented robust value iteration over a belief grid $\mathcal{G}\subset\Delta(\mathcal{Z})$ solves $|\mathcal{S}||\mathcal{G}||\mathcal{A}|$ LPs of size $O(|\mathcal{S}|)$ per sweep. For large type spaces a particle/variational posterior approximates $b$, and \eqref{eq:dual} keeps the inner problem tractable regardless of belief representation; the continuous-state case (a Lipschitz-critic realization of \eqref{eq:dual}) we leave as an explicit open item (\S\ref{sec:outlook}).

\section{From KL to Wasserstein: A Coherent-Risk Reading}
\label{sec:wevar}

EVaR is a KL-ball worst case \eqref{eq:evardual}; RATTL uses a Wasserstein ball. This is deliberate. Under KL the adversary cannot place mass where the nominal has none ($D_{\mathrm{KL}}=+\infty$ off-support), so precisely when the agent grows confident---and $\nominal$ concentrates away from rare catastrophes---the KL adversary \emph{loses} the ability to model the catastrophe. Wasserstein prices perturbations by physical distance, letting the adversary reach any state at proportional cost.

This is not merely rhetorical. Table~\ref{tab:klw} reports the inner worst case on a $5$-state ``cliff'' transition where the catastrophic state carries \emph{zero} nominal mass. The KL adversary cannot place \emph{any} mass on the catastrophe (it lies off the nominal support, where $D_{\mathrm{KL}}=+\infty$), so it cannot price the tail event even though it still perturbs the on-support mass; Wasserstein transports mass onto the catastrophe at finite cost and reports a worst case two orders of magnitude lower.

\begin{table}[t]
\centering
\small
\resizebox{\linewidth}{!}{%
\begin{tabular}{@{}lccc@{}}
\toprule
 & $\Wass$ & TV & KL \\
\midrule
Worst-case $\E_Q[V]$, $\varepsilon{=}0.25$ & $-345.0$ & $-217.5$ & $-3.4$ \\
Mass moved onto catastrophe & $0.40$ & $0.25$ & $0.00$ \\
\bottomrule
\end{tabular}}
\caption{Inner adversary at radius $\varepsilon{=}0.25$ on a cliff transition whose catastrophe carries \emph{zero} nominal mass (nominal value $+57.5$). KL cannot place any mass on the off-support catastrophe ($D_{\mathrm{KL}}{=}{+}\infty$), so its mass there stays $0$ and it cannot price the tail even as it reweights on-support mass; only Wasserstein reaches the catastrophe.}
\label{tab:klw}
\end{table}

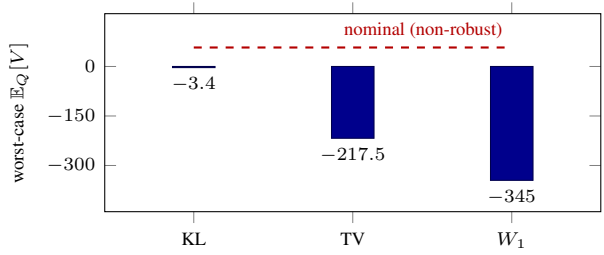
\begin{figure}[t]
\centering
\begin{tikzpicture}
\begin{axis}[width=0.95\linewidth,height=4.2cm,
  ybar, bar width=16pt, enlarge x limits=0.28,
  symbolic x coords={KL,TV,W1}, xtick=data, xticklabels={KL,TV,$\Wass$},
  ymin=-440,ymax=160, ytick={-300,-150,0},
  ylabel={worst-case $\E_Q[V]$}, ylabel near ticks,
  tick label style={font=\scriptsize}, label style={font=\scriptsize},
  nodes near coords, nodes near coords style={font=\scriptsize}]
\addplot[fill=blue!55!black,draw=blue!30!black] coordinates {(KL,-3.4) (TV,-217.5) (W1,-345.0)};
\draw[red!70!black,dashed,thick] (axis cs:KL,57.5) -- (axis cs:W1,57.5);
\node[font=\scriptsize,red!70!black,anchor=south east] at (axis cs:W1,57.5) {nominal (non-robust)};
\end{axis}
\end{tikzpicture}
\caption{The KL support catastrophe, visualized (radius $\varepsilon{=}0.25$). The KL adversary cannot reach the off-support catastrophe (its mass there stays $0$), so it barely departs from the non-robust value (dashed line) despite reweighting on-support mass; only Wasserstein transports mass onto the catastrophe and reports the true tail risk ($-345$).}
\label{fig:klbar}
\end{figure}

What coherent risk measure, then, does a Wasserstein ball induce? Assembling standard duality results, we record the following characterization.

\begin{theorem}[Coherence and Lipschitz-regularization]\label{thm:coh}
Fix finite $\mathcal{S}$ with a metric ground cost $d$, $\bar P\in\Delta(\mathcal{S})$, $\varepsilon\ge0$, and $\rho_\varepsilon(X):=\sup_{Q:\Wass(Q,\bar P)\le\varepsilon}\E_Q[X]$. Then \emph{(1)} $\rho_\varepsilon$ is a coherent risk measure (monotone, translation-equivariant, positively homogeneous, subadditive), being the Artzner et al.~\cite{artzner1999coherent} worst-case-expectation functional of the convex compact scenario set $\amb_\varepsilon$; and \emph{(2)} it equals a Lipschitz-regularized expectation,
\begin{equation}
\inf_{Q:\Wass(Q,\bar P)\le\varepsilon}\!\E_Q[f]=\sup_{\lambda\ge0}\big\{\E_{\bar P}[f_\lambda]-\lambda\varepsilon\big\},
\end{equation}
with $f_\lambda(s)=\min_{s'}(f(s')+\lambda d(s',s))$ the inf-convolution of $f$ with $\lambda d$, and the optimal $\lambda^\star\le\mathrm{Lip}_d(f)$~\cite{gao2023wasserstein,mohajerin2018wasserstein}.
\end{theorem}
\begin{proof}[Proof]
(1) $\amb_\varepsilon$ is nonempty ($\bar P\in\amb_\varepsilon$), convex (the map $Q\mapsto\Wass(Q,\bar P)$ is convex, by averaging optimal couplings), and compact (it is a closed subset of the simplex, $\Wass(\cdot,\bar P)$ being continuous via Kantorovich duality). The four axioms follow from properties of a supremum of linear functionals over a \emph{fixed} convex set; the Artzner representation is then immediate. (2) is finite-$\mathcal{S}$ LP strong duality: the transport LP $\min_\pi\sum\pi(s,s')f(s')$ s.t.\ first marginal $\bar P$ and $\sum\pi\,d\le\varepsilon$ has Lagrangian dual $\max_{\lambda\ge0}\{\E_{\bar P}[f_\lambda]-\lambda\varepsilon\}$, strictly feasible hence no gap. Details in Appendix~\ref{app:wevar}.
\end{proof}

On the canonical safety instance---a good state $g$ and a catastrophe $f$---the identity is sharp and, strikingly, is a CVaR, not an EVaR.

\begin{proposition}[Two-point Wasserstein risk is CVaR]\label{prop:cvar}
Let $\mathcal{S}=\{g,f\}$ with $V(g)=v_g>v_f=V(f)$, nominal catastrophe mass $q\in(0,1)$, $d(g,f)=D$, loss $L=-V$. For $\varepsilon<(1-q)D$ the worst-case value is
\begin{equation}
\underline V(\varepsilon)=\big(1-p^\star\big)v_g+p^\star v_f,\quad p^\star=q+\tfrac{\varepsilon}{D},
\end{equation}
and the associated robust loss is exactly a Conditional Value-at-Risk,
\begin{equation}
-\underline V(\varepsilon)=\CVaR_{\theta(\varepsilon)}(L),\qquad \theta(\varepsilon)=\frac{q}{q+\varepsilon/D}.
\end{equation}
Since $\CVaR\le\EVaR$ at a matched tail level, the Wasserstein risk is in turn dominated by an EVaR; the matching EVaR level depends on $\varepsilon$ with no closed form, which is why CVaR---not EVaR---is the clean object here.
\end{proposition}
\begin{proof}[Proof]
The two-point ball is $\{p:|p-q|\le\varepsilon/D\}$; $\E_Q[V]$ decreases in the catastrophe mass $p$, so the worst case is $p^\star=\min(1,q+\varepsilon/D)$. Substituting $p^\star=q/\theta$ into the two-point $\CVaR_\theta(L)=\tfrac1\theta(qL(f)+(\theta-q)L(g))$ recovers $p^\star L(f)+(1-p^\star)L(g)=-\underline V(\varepsilon)$ as an exact algebraic identity. The EVaR bound is the $\CVaR\le\EVaR$ ordering~\cite{ahmadi2012evar}. Details in Appendix~\ref{app:wevar}.
\end{proof}

Because $\varepsilon=\beta\Hent(b)$, the tail level $\theta(b)=q/(q+\beta\Hent(b)/D)$ \emph{decreases as belief entropy grows}: the entropy dial is literally a CVaR-tail dial (Figure~\ref{fig:tail}). This is a concrete, exact instance of the paper's slogan.

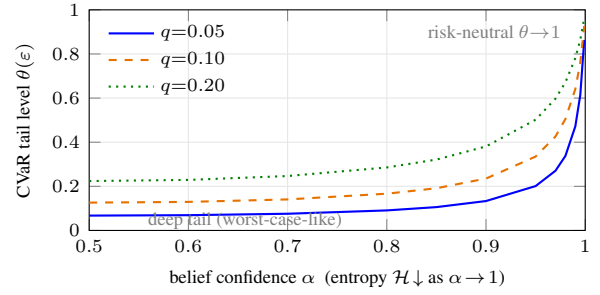
\begin{figure}[t]
\centering
\begin{tikzpicture}
\begin{axis}[width=0.95\linewidth,height=4.5cm,
  xlabel={belief confidence $\alpha$ \ (entropy $\Hent\!\downarrow$ as $\alpha\!\to\!1$)},
  ylabel={CVaR tail level $\theta(\varepsilon)$}, ylabel near ticks,
  xmin=0.5,xmax=1.0,ymin=0,ymax=1.0,
  tick label style={font=\scriptsize}, label style={font=\scriptsize},
  legend style={font=\scriptsize,at={(0.02,0.98)},anchor=north west,draw=none,fill=none},
  grid=major, grid style={gray!20}]
\addplot[blue,thick] coordinates {(0.500,0.0673) (0.600,0.0692) (0.700,0.0757) (0.800,0.0908) (0.850,0.1058) (0.900,0.1333) (0.950,0.2012) (0.970,0.2706) (0.980,0.3377) (0.990,0.4717) (0.995,0.6137) (0.999,0.8634)};
\addlegendentry{$q{=}0.05$}
\addplot[orange!90!black,thick,dashed] coordinates {(0.500,0.1261) (0.600,0.1294) (0.700,0.1407) (0.800,0.1666) (0.850,0.1913) (0.900,0.2352) (0.950,0.3350) (0.970,0.4260) (0.980,0.5050) (0.990,0.6410) (0.995,0.7606) (0.999,0.9267)};
\addlegendentry{$q{=}0.10$}
\addplot[green!50!black,thick,dotted] coordinates {(0.500,0.2239) (0.600,0.2291) (0.700,0.2467) (0.800,0.2856) (0.850,0.3212) (0.900,0.3809) (0.950,0.5019) (0.970,0.5975) (0.980,0.6711) (0.990,0.7812) (0.995,0.8640) (0.999,0.9620)};
\addlegendentry{$q{=}0.20$}
\node[font=\scriptsize,gray,anchor=east] at (axis cs:0.985,0.90) {risk-neutral $\theta{\to}1$};
\node[font=\scriptsize,gray,anchor=west] at (axis cs:0.55,0.045) {deep tail (worst-case-like)};
\end{axis}
\end{tikzpicture}
\caption{Entropy is a CVaR dial (Proposition~\ref{prop:cvar}). As the belief sharpens ($\alpha\!\to\!1$, $\Hent\!\to\!0$), the induced CVaR tail level $\theta(\varepsilon)=q/(q+\beta\Hent/D)$ rises from a deep tail (worst-case-like caution) toward $1$ (risk-neutrality), for catastrophe masses $q\in\{0.05,0.1,0.2\}$. This is the risk spectrum of Figure~\ref{fig:spectrum}, now traversed automatically by learning.}
\label{fig:tail}
\end{figure} Two honest caveats: for $|\mathcal{S}|\ge3$ the worst case spreads mass to the nearest low-value states, giving a $d$-weighted ``transport-CVaR'' rather than ordinary CVaR; and a single, $\varepsilon$-uniform EVaR-level identity does \emph{not} hold (the matching EVaR level depends transcendentally on $\varepsilon$). The general transport-CVaR$\leftrightarrow$EVaR comparison is the headline open problem; Proposition~\ref{prop:cvar} settles the canonical case.

\section{Worked Example: The Ambiguous Bridge}
\label{sec:bridge}

\noindent\emph{(All curves below are closed-form evaluations of the robust backup, not learning runs; here $\alpha=b(\text{benign})$ denotes belief confidence.)}

To make every quantity concrete we instantiate RATTL on a diagnostic with three states $\{\Sbr,\SG,\SF\}$ (bridge, goal, fall), two actions $\{\textsc{Sprint},\textsc{Crawl}\}$, and two types $\{\text{benign},\text{adversarial}\}$. \textsc{Sprint} reaches $\SG$ under the benign type but falls to $\SF$ under the adversarial type; \textsc{Crawl} always reaches $\SG$. Rewards: $r(\Sbr,\textsc{Sprint}){=}{-}1$, $r(\Sbr,\textsc{Crawl}){=}{-}20$, $r(\SG){=}{+}100$, $r(\SF){=}{-}1000$; ground distance $D{=}d(\SG,\SF){=}1$; $\beta{=}1$. Writing $\alpha=b(\text{benign})$ and $\varepsilon(\alpha)=\Hent(\alpha)$, both successor states are terminal, so a single robust backup with the two-point transport $\Wass(\mu,\nu){=}D|\mu_1-\nu_1|$ gives closed-form robust $Q$-values:
\begin{align}
Q(\textsc{Sprint},\alpha) &= -1001 + 1100\max\!\big(0,\alpha-\varepsilon(\alpha)\big),\\
Q(\textsc{Crawl},\alpha) &= 80 - 1100\min\!\big(\varepsilon(\alpha),1\big).
\end{align}
Setting them equal, the $\varepsilon$ terms cancel and $1100\alpha=1081$, so the safety switch is at
\begin{equation}
\alpha^* = \tfrac{1081}{1100}\approx 0.983 .
\end{equation}
The agent \textsc{Crawl}s until it is $\sim$98\% confident the conditions are benign, then \textsc{Sprint}s (Figure~\ref{fig:q}, Table~\ref{tab:bridge}). In this symmetric environment the threshold is determined purely by the reward asymmetry and is \emph{independent} of $\beta$ (both actions face the same per-unit transport penalty); $\beta$ controls the conservatism of the \emph{value} and shifts the threshold only when the safe and risky actions face different ambiguity. Well-posedness requires $\varepsilon(b)\le D$, i.e.\ $\beta<D/\ln|\mathcal{Z}|\approx1.443$.

\begin{table}[t]
\centering\small
\begin{tabular}{@{}lcccc@{}}
\toprule
$\alpha$ & $\varepsilon(\alpha)$ & $Q(\textsc{Sprint})$ & $Q(\textsc{Crawl})$ & $\pi^*$\\
\midrule
$0.50$ & $0.6931$ & $-1001.00$ & $-682.46$ & \textsc{Crawl}\\
$0.85$ & $0.4227$ & $-530.98$  & $-384.98$ & \textsc{Crawl}\\
$0.99$ & $0.0560$ & $+26.40$   & $+18.40$  & \textsc{Sprint}\\
\midrule
$\alpha^*{=}1081/1100$ & $0.0872$ & $-15.95$ & $-15.95$ & switch\\
\bottomrule
\end{tabular}
\caption{Robust value iteration on the Ambiguous Bridge ($\beta{=}1$, $D{=}1$). The safety floor is $\Vmm(0.5)=80-1100\ln2=-682.46$; the best-response ceiling is $\Vbr=99$ as $\alpha\to1$. The threshold $\alpha^*=1081/1100$ is $\beta$-independent here.}
\label{tab:bridge}
\end{table}

\begin{figure}[t]
\centering
\begin{tikzpicture}
\begin{axis}[width=0.95\linewidth,height=4.7cm,
  xlabel={belief confidence $\alpha=b(\text{benign})$},
  ylabel={robust $Q$-value}, ylabel near ticks,
  xmin=0.5,xmax=1.0,ymin=-1050,ymax=150,
  legend style={font=\scriptsize,at={(0.02,0.98)},anchor=north west,draw=none,fill=none},
  tick label style={font=\scriptsize}, label style={font=\scriptsize},
  grid=major, grid style={gray!20}]
\addplot[blue,thick] coordinates {
(0.5000,-1001.00) (0.5500,-1001.00) (0.6000,-1001.00) (0.6500,-998.19) (0.7000,-902.95) (0.7500,-794.57) (0.8000,-671.44) (0.8300,-589.47) (0.8500,-530.98) (0.8800,-436.62) (0.9000,-368.59) (0.9200,-295.65) (0.9400,-216.66) (0.9600,-129.74) (0.9700,-82.22) (0.9750,-57.10) (0.9800,-30.84) (0.9827,-16.10) (0.9850,-3.17) (0.9900,26.40) (0.9950,58.87) (0.9990,89.20)};
\addlegendentry{$Q(\textsc{Sprint})$}
\addplot[red,thick,dashed] coordinates {
(0.5000,-682.46) (0.5500,-676.95) (0.6000,-660.31) (0.6500,-632.19) (0.7000,-591.95) (0.7500,-538.57) (0.8000,-470.44) (0.8300,-421.47) (0.8500,-384.98) (0.8800,-323.62) (0.9000,-277.59) (0.9200,-226.65) (0.9400,-169.66) (0.9600,-104.74) (0.9700,-68.22) (0.9750,-48.60) (0.9800,-27.84) (0.9827,-16.07) (0.9850,-5.67) (0.9900,18.40) (0.9950,45.37) (0.9990,71.30)};
\addlegendentry{$Q(\textsc{Crawl})$}
\draw[gray,dotted] (axis cs:0.9827,-1050) -- (axis cs:0.9827,150);
\node[font=\scriptsize,gray] at (axis cs:0.93,90) {$\alpha^*{\approx}0.983$};
\end{axis}
\end{tikzpicture}
\caption{The safety switch. \textsc{Crawl} dominates while uncertain; at $\alpha^*\approx0.983$ the risky action's \emph{worst-case} value overtakes the safe one and the agent switches to \textsc{Sprint}.}
\label{fig:q}
\end{figure}
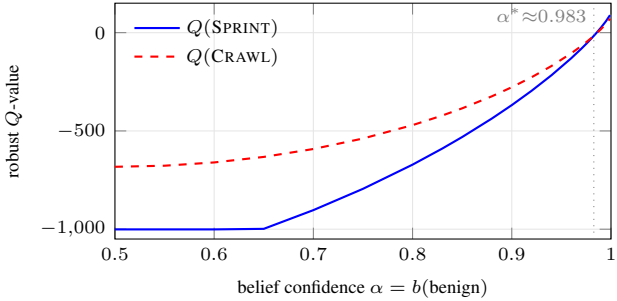

\begin{figure}[t]
\centering
\begin{tikzpicture}
\begin{axis}[width=0.95\linewidth,height=4.4cm,
  xlabel={belief confidence $\alpha$}, ylabel={$V^*(\Sbr,\alpha)$}, ylabel near ticks,
  xmin=0.5,xmax=1.0,ymin=-750,ymax=150,
  tick label style={font=\scriptsize}, label style={font=\scriptsize},
  legend style={font=\scriptsize,at={(0.02,0.5)},anchor=west,draw=none,fill=none},
  grid=major, grid style={gray!20}]
\addplot[black,very thick] coordinates {
(0.5000,-682.46) (0.6000,-660.31) (0.7000,-591.95) (0.8000,-470.44) (0.8500,-384.98) (0.9000,-277.59) (0.9400,-169.66) (0.9700,-68.22) (0.9827,-16.07) (0.9850,-3.17) (0.9900,26.40) (0.9950,58.87) (0.9990,89.20)};
\addlegendentry{RATTL $V^*$}
\addplot[red,dashed,thick] coordinates {(0.5,-682.46) (1.0,-682.46)};
\addlegendentry{worst-case floor $-682.5$}
\addplot[green!50!black,dotted,thick] coordinates {(0.5,99) (1.0,99)};
\addlegendentry{best-response ceiling $99$}
\end{axis}
\end{tikzpicture}
\caption{The Safety Sandwich, instantiated. RATTL's value rises from the worst-case floor ($\Vmm$, never-shrinking ambiguity) toward the full-information ceiling ($\Vbr$) as the belief sharpens---never less safe than maximin, never more reckless than the informed optimum. The two reference lines are \emph{constant} bounds ($-682.5$ and $99$); $V^*$ provably stays between them at every belief.}
\label{fig:sandwich}
\end{figure}
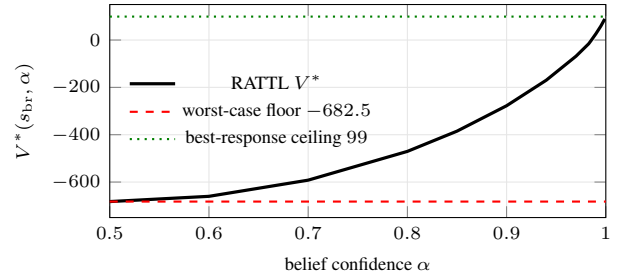

\section{Related Work}
RATTL sits at the intersection of robust MDPs, risk-sensitive RL, and Bayesian opponent modeling. Robust MDPs with static, rectangular uncertainty sets~\cite{iyengar2005robust,nilim2005robust,wiesemann2013robust} are RATTL's ancestor; our radius is belief-dependent and non-stationary, and we go in the opposite direction to the RMDP$\leftrightarrow$game equivalence~\cite{chatterjee2024robust}. On the risk side, EVaR~\cite{ahmadi2012evar} and CVaR~\cite{rockafellar2000cvar} supply the coherent backbone~\cite{artzner1999coherent}; Ni \& Bhat~\cite{ni2024evar} show stationary policies suffice for EVaR total-reward MDPs---justifying RATTL's policy class---but at a \emph{fixed} risk level, whereas ours is belief-selected. The closest competitors couple ambiguity to Bayesian posteriors: Russel \& Petrik~\cite{russel2019tight} adapt sets to the \emph{policy} (not the belief); Choi and Li~\cite{choi2025bayesian} contract interval \emph{credible sets} rather than entropy-modulated Wasserstein balls and lack a risk-measure reading; Nakao et al.~\cite{nakao2025drpomdp} study DR-POMDPs with \emph{static} distance-based ambiguity. Derman \& Mannor~\cite{derman2020distributional} relate Wasserstein DRO to value regularization, which Theorem~\ref{thm:coh} sharpens into an explicit coherent-risk identity. To our knowledge, no prior work couples a Wasserstein radius to \emph{Shannon belief entropy} (as opposed to posterior credible-set width, cf.\ Choi and Li), nor derives the resulting Safety Sandwich and two-point CVaR identity.

\section{Limitations, Future Work, and Conclusion}
\label{sec:outlook}

\paragraph{Limitations.} We are explicit about the current scope. (a) \emph{Tabular.} The exact algorithm discretizes the belief simplex and is feasible only for small $|\mathcal{Z}|$; we give a scalable roadmap (\S\ref{sec:theory}) but no large-scale empirical validation yet. (b) \emph{Added conditions.} The guarantees need more than bare properness: uniform reachability (Assumption~\ref{ass:unif}) for the contraction over the belief continuum, and persistent identification/excitation for convergence and its rate---degenerate exploration is excluded. (c) \emph{Two points.} The exact CVaR identity (Proposition~\ref{prop:cvar}) is proven only for the canonical two-point catastrophe; for $|\mathcal{S}|\ge3$ it becomes a $d$-weighted ``transport-CVaR'' and the EVaR comparison is only one-sided. (d) \emph{Calibration.} $\beta$ must satisfy $\beta<D/\ln|\mathcal{Z}|$, and the threshold's $\beta$-independence is specific to symmetric environments like the Bridge.

\paragraph{Future work.} (i) The general transport-CVaR$\leftrightarrow$EVaR comparison for $|\mathcal{S}|\ge3$. (ii) Hybrid Wasserstein$+$KL ambiguity combining reachability with tail sensitivity. (iii) R\'enyi/Tsallis entropy as alternative dials---does the Safety Sandwich survive for any concave, zero-at-point-mass information measure? (iv) Sample complexity of belief-adaptive robust RL, where the radius shrinks at rate $O(\log t/t)$. (v) Scaling the dual~\eqref{eq:dual} to continuous spaces via Lipschitz critics. (vi) Multi-agent risk composition under private beliefs. The headline applied direction---and the workshop motivation---is runtime safety for \emph{agentic GenAI}: agents that plan, retrieve, and invoke tools hold a Bayesian or surrogate belief over latent context and face deployment-time distribution shift and adversarial tail risk, exactly RATTL's structure. Belief-entropy-modulated robustness then offers a principled dial for abstention and graduated conservatism, with the Safety Sandwich providing a value-level bracket---a worst-case floor and a Bayesian-best-response ceiling---rather than a runtime behavioral guarantee; turning it into an actionable runtime certificate, and validating it on real LLM-agent pipelines, is the main empirical goal.

\paragraph{Conclusion.} RATTL makes an agent's risk attitude a function of its epistemic state by tying a Wasserstein ambiguity radius to belief entropy. We proved a contraction theorem, a Safety Sandwich bracketing the value between a worst-case floor and a Bayesian-best-response ceiling, and almost-sure convergence to best response, each with explicit, adversarially-checked conditions. We took a first step on the risk-measure identity of Wasserstein ambiguity: it is coherent and equals a Lipschitz-regularized expectation in general, and an exact CVaR with an entropy-controlled tail level on the canonical catastrophe---so the belief literally selects a position on the coherent-risk spectrum. The Ambiguous Bridge makes the resulting safety switch sharp and interpretable. The thesis is simple: robustness should decrease with epistemic certainty, and belief entropy is the principled dial that achieves it.

\section*{Acknowledgments}
We thank the RobustifAI workshop reviewers for constructive feedback that substantially shaped this extended version.

\appendix
\section{Full Proof of Theorem~\ref{thm:contraction}}\label{app:contraction}
Work on the affine space $\mathcal{B}_{\mathcal{T}}=\{V\ \text{bounded}:V(s,\cdot)=r(s)\ \forall s\in\mathcal{T}\}$, on which differences vanish on $\mathcal{T}$ so $\|\cdot\|_w$ separates points and is equivalent to $\|\cdot\|_\infty$ over $s\notin\mathcal{T}$; $\mathcal{B}_{\mathcal{T}}$ is complete and $\mathbb{T}:\mathcal{B}_{\mathcal{T}}\to\mathcal{B}_{\mathcal{T}}$. Iterating Assumption~\ref{ass:unif} over blocks of $m$ steps gives $\Pr(\tau_{\mathcal{T}}>km)\le(1-\eta)^k$, hence $w(s)=\E[\tau_{\mathcal{T}}]\le\sum_{k\ge0} m(1-\eta)^k=m/\eta=:W$ uniformly in $b$ and strategy, and $w(s)\ge1$ since a non-terminal state needs a transition. For bounded $f,g$ on any set, $f\le g+\sup|f-g|$ gives $\inf f\le\inf g+\sup|f-g|$; symmetrizing yields $|\inf_Q f-\inf_Q g|\le\sup_Q|f-g|$. With $f(Q)=\sum_{s'}Q(s')V(s',\psi)$, $g(Q)=\sum_{s'}Q(s')V'(s',\psi)$ and $Q\in\Delta(\mathcal{S})$, the inner terms differ by at most $\sup_{Q\in\amb(b)}\sum_{s'\notin\mathcal{T}}Q(s')|V-V'|(s',\psi)$. Using $|V-V'|(s',\psi)\le\|V-V'\|_w\,w(s')$ on $s'\notin\mathcal{T}$ (and $0$ on $\mathcal{T}$), cancellation of $r(s,a)$, and non-expansiveness of $\max_a$,
\[
|(\mathbb{T}V-\mathbb{T}V')(s,b)|\le\|V-V'\|_w\!\!\sup_{Q\in\amb(b\mid s,a)}\!\!\sum_{s'\notin\mathcal{T}}\!Q(s')w(s').
\]
The worst-case expected hitting time satisfies the unit-cost SSP drift inequality $\sup_{a,b,Q}\sum_{s'\notin\mathcal{T}}Q(s')w(s')\le w(s)-1$ (one admissible step costs $1$ and leaves residual worst-case time $\le w(s')$; \cite{bertsekas1996neuro}). Hence $|(\mathbb{T}V-\mathbb{T}V')(s,b)|\le\|V-V'\|_w(w(s)-1)$; dividing by $w(s)\in[1,W]$ gives $\|\mathbb{T}V-\mathbb{T}V'\|_w\le(1-1/W)\|V-V'\|_w$. Banach's theorem completes the proof. $\square$

\section{Full Proof of Theorem~\ref{thm:sandwich}}\label{app:sandwich}
Both $\mathbb{T}$ and the frozen-radius $\mathbb{T}_{\max}$ are monotone self-maps of $\mathcal{B}_{\mathcal{T}}$ with unique fixed points $V^*,\Vmm$ (Theorem~\ref{thm:contraction} applies to each, with $\varepsilon_{\max}$ a valid radius), and value iteration converges for both.
\emph{Lower bound.} $0\le\Hent(b)\le\ln|\mathcal{Z}|$, so $\varepsilon(b)\le\varepsilon_{\max}$ with common center, giving $\amb(b)\subseteq\amb_{\max}(b)$ and $\inf_{\amb_{\max}}\le\inf_{\amb}$; adding $r$ and taking $\max_a$, $\mathbb{T}_{\max}V\le\mathbb{T}V$ for all $V$. Induct: $\mathbb{T}_{\max}^0V^*=V^*$; if $\mathbb{T}_{\max}^kV^*\le V^*$ then by monotonicity $\mathbb{T}_{\max}^{k+1}V^*\le\mathbb{T}_{\max}V^*\le\mathbb{T}V^*=V^*$. Limit: $\Vmm\le V^*$.
\emph{Upper bound.} Fix a policy $\pi$. (a) $\nominal\in\amb(b)$ (zero transport), so the robust value $V^\pi_{\mathrm{rob}}\le V^\pi_{\mathrm{nom}}$, the value under the nominal kernel with $\psi$-updates. (b) The joint law of $(z,s')$ with $z\sim b$, $s'\sim P(\cdot\mid s,a,z)$ has $s'$-marginal $\nominal(\cdot\mid s,a)$ and $z$-posterior $\psi(b,s,a,s')$; by induction on $t$ the nominal process and the mixture process ``draw $z\sim b$ once, then follow $P(\cdot\mid\cdot,\cdot,z)$'' induce the same trajectory law and maintain $b_t=\Pr^{\mathrm{mix}}(z\mid h_t)$. Since the return depends only on the observable trajectory, $V^\pi_{\mathrm{nom}}(s,b)=\E^{\mathrm{mix}}[G]=\sum_z b(z)V^\pi_z(s)$ by the tower property. (c) $V^\pi_z(s)\le\Vbr(s,z)$ since $\pi$ is feasible in the type-$z$ MDP. Chaining and taking $\sup_\pi$ (with $V^*=\sup_\pi V^\pi_{\mathrm{rob}}$ from proper minimax SSP theory) gives $V^*(s,b)\le\sum_z b(z)\Vbr(s,z)$. $\square$

\section{Statement and Proof of Theorem~\ref{thm:convergence}}\label{app:convergence}
\emph{Consistency.} Under (PI) an identifying $(s,a)$ for each type pair is visited infinitely often, so the realized log-likelihood ratio of $z^*$ against any $z\ne z^*$ diverges to $+\infty$ a.s.\ (a sum of i.o.\ strictly positive-mean, bounded-variance increments), giving $b_t\to\delta_{z^*}$ a.s.~\cite{schwartz1965bayes}; thus $\Hent(b_t)\to0$ and $\varepsilon(b_t)\to0$, and $\amb(b_t\mid s,a)\to\{P(\cdot\mid s,a,z^*)\}$ in the $\Wass$-Hausdorff metric.
\emph{Continuity.} Let $N$ be a neighborhood of $\delta_{z^*}$ on which the Bayes normalizer is $\ge p_{\min}>0$, so $\psi(\cdot,s,a,s')$ is $L_\psi$-Lipschitz on $N$. The class $\mathcal{C}_L=\{V\in\mathcal{B}_{\mathcal{T}}: |V(s,b)-V(s,b')|\le L\,\|b-b'\|\ \forall b,b'\in N\}$ is closed and $\mathbb{T}$-invariant once $L$ is large enough that the per-step belief-modulus contraction $\,(1-1/W)L_\psi<1\,$ holds (the only place this extra condition enters); the unique fixed point $V^*$ therefore lies in $\mathcal{C}_L$ and is continuous at $\delta_{z^*}$. Combined with $b_t\to\delta_{z^*}$ and $\amb(b_t)\to\{P_{z^*}\}$, $V^*(s,b_t)\to\Vbr(s,z^*)$ a.s.
\emph{Rate.} Under persistent excitation the posterior on each wrong type decays geometrically in the number of identifying visits, and with $\Theta(t)$ such visits $\E[\Hent(b_t)]=O(|\mathcal{Z}|\log t/t)$; since $|\Vbr(s,z^*)-V^*(s,b_t)|\le C\sum_{z\ne z^*}b_t(z)$ for a constant $C$ (both the floor gap and the averaged-ceiling gap vanish with the residual mass on wrong types), the price of robustness is $O(\log t/t)$. $\square$

\section{Details for \S\ref{sec:wevar}}\label{app:wevar}
Convexity of $\amb_\varepsilon$: for optimal couplings $\pi_0,\pi_1$ of $Q_0,Q_1$ with $\bar P$, the coupling $\theta\pi_0+(1-\theta)\pi_1$ has marginals $\theta Q_0+(1-\theta)Q_1$ and $\bar P$, so $\Wass(\theta Q_0+(1-\theta)Q_1,\bar P)\le\theta\Wass(Q_0,\bar P)+(1-\theta)\Wass(Q_1,\bar P)$. Coherence axioms follow because $\amb_\varepsilon$ does not depend on $X$: monotonicity and positive homogeneity are termwise; translation-equivariance uses $\langle Q,\mathbf1\rangle=1$; subadditivity uses that the two suprema decouple. The duality is finite-LP strong duality with multipliers $u(s)$ (marginals) and $\lambda\ge0$ (budget); optimizing $u(s)=\min_{s'}(f(s')-\lambda d(s,s'))=f_\lambda(s)$ gives the boxed form, and $f_\lambda=f$ once $\lambda\ge\mathrm{Lip}_d(f)$, bounding $\lambda^\star$. For Proposition~\ref{prop:cvar}, the two-point ball is $\{p:|p-q|\le\varepsilon/D\}$ by Kantorovich duality ($|h(g)-h(f)|\le D$), the worst case is the right endpoint $p^\star$, and substituting $p^\star=q/\theta$ (with $\theta=q/(q+\varepsilon/D)$) into $\CVaR_\theta(L)=\tfrac1\theta(qL(f)+(\theta-q)L(g))$ reproduces $-\underline V(\varepsilon)$ exactly; the $\CVaR\le\EVaR$ ordering~\cite{ahmadi2012evar} gives the one-sided EVaR bound. $\square$

\bibliographystyle{named}
\bibliography{rattl_extended}

\end{document}